\documentclass[11pt]{article}

\usepackage[preprint]{acl}

\usepackage{times}
\usepackage{latexsym}
\usepackage[T1]{fontenc}
\usepackage[utf8]{inputenc}
\usepackage{microtype}
\usepackage{inconsolata}
\usepackage{graphicx}

\usepackage[skins]{tcolorbox}

\usepackage{tikz}
\usepackage{xcolor}

\usepackage{amsmath,amssymb,amsthm}
\usepackage{mathtools}

\usepackage{bbding}
\usepackage{booktabs}
\usepackage{array}

\usepackage{url}

\definecolor{textgray}{RGB}{50,50,50}
\definecolor{bordercolor}{RGB}{180,180,180}

\definecolor{lightgraybg}{HTML}{F7F7F7}
\definecolor{leftbar}{HTML}{E74C3C}
\definecolor{validgreen}{HTML}{2ECC71}

\newtcolorbox{samplebox}{
  enhanced,
  colback=white,
  boxrule=0.75pt,
  arc=2pt,
  left=4pt,
  right=4pt,
  top=4pt,
  bottom=4pt,
  fontupper=\small
}

\definecolor{comp1}{HTML}{4A90D9}
\definecolor{comp2}{HTML}{E8913A}
\definecolor{comp3}{HTML}{D35F5F}
\definecolor{comp4}{HTML}{9B59B6}
\definecolor{optcomp}{HTML}{45B5AA}

\definecolor{plotblue}{RGB}{70, 130, 180}
\definecolor{plotred}{RGB}{255, 1, 1}
\definecolor{plotdarkred}{RGB}{139, 0, 0}
\definecolor{plotgold}{RGB}{255, 215, 0}

\newcommand{\badge}[2][red]{%
  \tikz[baseline={([yshift=-1pt]b.base)}]
    \node[circle, fill=#1, text=white, inner sep=1pt,
    minimum size=1.0em, font=\bfseries\scriptsize] (b) {#2};%
}

\newtheorem{proposition}{Proposition}

\title{UFAL-CUNI at SemEval-2026 Task 11: An Efficient Modular Neuro-symbolic Method for Syllogistic Reasoning}

\author{
  Ivan Kartáč \quad Kristýna Onderková \quad Jan Bronec \\
  \textbf{Zdeněk Kasner} \quad \textbf{Mateusz Lango} \quad \textbf{Ondřej Dušek} \\[2pt]
  Institute of Formal and Applied Linguistics \\
  Faculty of Mathematics and Physics, Charles University \\[2pt]
  \texttt{\{kartac,onderkova,bronec,kasner,lango,odusek\}@ufal.mff.cuni.cz}
}

\def\pz{\phantom{0}}

\begin{document}
\maketitle
\begin{abstract}

This paper describes our system submitted to SemEval-2026 Task 11: Disentangling Content and Formal Reasoning in Large Language Models. We present an efficient modular neuro-symbolic approach, combining a symbolic prover with small reasoning LLMs (4B parameters). The system consists of an LLM-based parser that translates natural language syllogisms to a first-order logic (FOL) representation, an automated theorem prover, and two optional modules: machine translation for multilingual inputs and a symbolic retrieval component for the identification of relevant premises. The system achieves competitive accuracy and relatively low content effect on most subtasks. Our ablations show that this approach outperforms LLM-based zero-shot baselines in this parameter size range, but also reveal limited multilingual capabilities of small LLMs. Finally, we include a discussion of the task's main ranking metric and analyze its limitations.\footnote{The code and data are available at \url{https://github.com/ivankartac/SemEval-2026_task11}} 
\end{abstract}

\section{Introduction}
\label{sec:introduction}

\begin{figure*}[t]
    \centering
    \includegraphics[width=\textwidth]{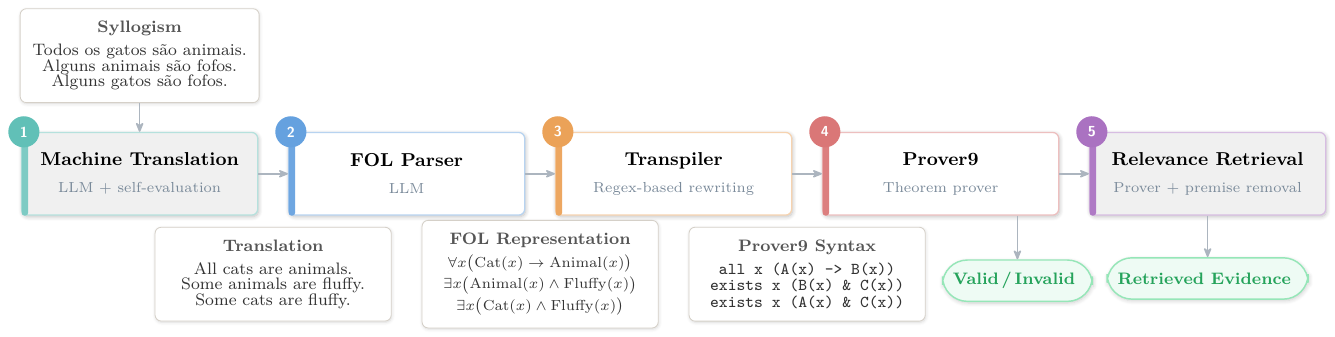}
    \caption{Overview of our system: \badge[optcomp]{1} \emph{Machine Translation} translates multilingual inputs to English (for subtasks 3 and 4); \badge[comp1]{2} \emph{FOL Parser} translates natural language propositions into FOL formulas in LaTeX format; \badge[comp2]{3} \emph{Transpiler} rewrites LaTeX formulas to the target Prover9 syntax; \badge[comp3]{4} \emph{Prover9} determines the validity of the syllogism; \badge[comp4]{5} \emph{Relevance Retrieval} removes irrelevant premises by using the prover to check which premises do not influence validity (for subtasks 2 and 4).}
    \label{fig:pipeline}
\end{figure*}

Large Language Models (LLMs) have demonstrated strong performance on a number of reasoning tasks, including syllogistic reasoning. However, LLMs suffer from various reasoning biases (\citealp{eisape-etal-2024-systematic,lampinen2024language}), particularly the \emph{content effect} \citep{evans1983conflict}, where the decision about the validity of a given argument is influenced by its compatibility with world knowledge.
The SemEval-2026 Task 11 \cite{valentino-etal-2026-semeval} on multilingual syllogistic reasoning addresses this issue, with the goals of building systems with content-independent reasoning and advancing our understanding of the content effect. 

In our SemEval-2026 Task 11 submission, we explore how small LLMs can be used for syllogistic reasoning and first-order logic (FOL) formalization and made robust to content effects. We propose an efficient neuro-symbolic approach, combining small reasoning LLMs (4B parameters) with a symbolic prover. 
Specifically, we first translate each natural language proposition into a FOL representation and then use an automated theorem prover to determine whether the conclusion is valid. Optionally, we use an additional LLM as a machine translation component for multilingual inputs, while the identification of relevant premises is addressed symbolically by the FOL prover.

Unlike existing neuro-symbolic approaches, we utilize a logic representation well-represented in the training data of the LLMs, rather than instructing them to provide formalizations directly in the target syntax. Specifically, we instruct the LLMs to first formalize the input in LaTeX notation, followed by a rule-based translation to the target syntax. We show that using this intermediate format leads to more reliable, higher-quality formal representations.

We apply our system to all four subtasks of SemEval-2026 Task 11, achieving an accuracy of around 95\% for most of them while keeping the content effect relatively low. Based on our ablation experiments, we find that our method is able to significantly reduce the content effect for LLMs of this size, in contrast to instructing them to directly reason about the validity of the syllogisms. To better understand the failure modes of the system, we present a detailed error analysis.

Finally, we present an empirical and theoretical analysis of the main ranking metric and show that it suffers from limited robustness.

\section{Related Work}
\label{sec:related_work}

\paragraph{Content effects in reasoning} Biases in human reasoning have been extensively studied in psychology \citep{tversky1974judgment}, including content effects in syllogistic reasoning \citep{evans1983conflict,klauer2000belief}. Recently, reasoning biases similar to those found in humans have been demonstrated for LLMs (\citealp{saparov2023languagemodelsgreedyreasoners,eisape-etal-2024-systematic,lampinen2024language}). \citet{kim-etal-2025-reasoning} use circuit discovery to show that while LLMs learn content-independent reasoning mechanisms, these are entangled with world knowledge and therefore prone to content effects. This issue has been addressed through various approaches, such as activation steering \citep{valentino2025mitigatingcontenteffectsreasoning}, prompting \citep{xu2024faithful}, supervised fine-tuning (\citealp{bertolazzi-etal-2024-systematic,zhou-etal-2025-dissecting}), or hybrid neuro-symbolic methods discussed below.

\paragraph{Neuro-symbolic approaches to logical reasoning} Recently, systems combining LLMs and symbolic provers have been proposed to address deductive reasoning. LINC \citep{olausson-etal-2023-linc} applies an LLM to translate natural language premises into FOL, followed by a symbolic prover to determine the validity of a conclusion. Logic-LM \citep{pan-etal-2023-logic} and Logic-LM++ \cite{kirtania-etal-2024-logic} translate the input into different symbolic formulations depending on the problem and apply a corresponding inference tool, such as FOL prover or SAT solver, followed by a refinement step in case the solver fails. 
\citet{quan-etal-2024-verification} combine LLMs with a theorem prover to automatically verify and improve natural language explanations used to evaluate models in natural language inference (NLI).
All these methods translate premises directly into the prover syntax, which can lead to suboptimal results, as we show in Section~\ref{sec:ablations}.

In contrast to these approaches, LogicGuide \citep{poesia2024certified} integrates a \emph{guide tool} within the LLM's decoding. The guide tool, based on the Peano theorem prover, computes a set of valid continuations in a given step, and the LLM selects one of them through constrained decoding.

\section{Task Description}
\label{sec:task_description}

SemEval-2026 Task 11 \cite{valentino-etal-2026-semeval} investigates how content influences language models when performing formal reasoning. The task evaluates models on the validity of Aristotelian syllogisms \cite{sep-aristotle-logic}. Those can align with or oppose common world knowledge, which allows us to measure the content effect on the reasoning \cite{evans1983conflict}. The task consists of four subtasks: (1) assessing validity of syllogisms, (2) selecting relevant premises for a conclusion, (3) assessing validity in multiple languages, and (4) selecting relevant premises in multiple languages.

\paragraph{Metrics}

The dataset samples are split into four groups depending on their logical validity and the real-world plausibility of their arguments.
The \emph{content effect} (CE) is the average of accuracy differences between valid and invalid samples and between plausible and implausible samples, where each accuracy difference is balanced with respect to the other feature \cite{valentino-etal-2026-semeval}.
In addition, standard \emph{validity accuracy} (Acc.) over all samples and the \emph{F1 score on premise relevance} (for subtasks 2 and 4) are measured.
Finally, the content effect is used to discount the overall accuracy and obtain the final combined score (CS):
\begin{equation}\label{eq:final_score}
    \text{CS} = \frac{\text{Acc}}{1 + \ln(1 + \text{CE})}
\end{equation}

\section{System Description}
\label{sec:system_description}

Our system uses a modular neuro-symbolic approach and consists of four main components: translator (for multilingual subtasks), FOL parser, transpiler, and a FOL prover (see Figure~\ref{fig:pipeline}).
\paragraph{Translator} For multilingual subtasks, we pre-translate the syllogisms from all source languages into English. We prompt an LLM to first provide a translation of a syllogism, followed by a self-evaluation step, and a potential correction step based on the self-evaluation feedback. Prompt templates for each step are shown in Figures~\ref{fig:prompt_translation}--\ref{fig:prompt_translation_feedback} in the Appendix.

\paragraph{FOL parser} We apply a small reasoning LLM to translate natural language syllogisms into FOL representations. To reduce the content effect, individual propositions are formalized one by one, each in a separate inference call. In each subsequent call, the LLM is provided with a mapping from natural language to FOL for the already formalized propositions to ensure consistent predicate names.
Since LaTeX notation is very common in training data for similar tasks, we hypothesize that LLMs could provide higher-quality FOL representations when explicitly instructed to generate them in LaTeX format rather than the target prover syntax.

\paragraph{Transpiler} To translate the FOL representations to the target prover syntax, we apply a regex-based rewriting with rules to translate the formalized syllogisms from LaTeX to the prover format.

\paragraph{FOL prover} To determine the validity of a formalized syllogism, we apply Prover9,\footnote{\url{https://www.cs.unm.edu/~mccune/prover9}} an automated theorem prover for first-order logic.

\paragraph{Relevant premise retrieval} In subtasks 2 and 4, we apply the following approach to identify relevant premises: We use a greedy algorithm that iterates through premises of a valid syllogism, drops each premise, and applies the prover to check if the proof still holds. If not, the premise is necessary for the conclusion and is labeled as relevant.

\section{Experiments}
\label{sec:experiments}

\subsection{Data}
\label{sec:experiments_data}

We use the first 500 training set examples supplied for subtask 1 as validation data and to construct additional synthetic validation sets for the other three sub-tasks. We synthesize multilingual data for subtasks 3 and 4 by translating the original validation set to target languages and sample irrelevant premises for subtasks 2 and 4 from the remaining examples in the training set. Appendix~\ref{sec:appendix_data} describes the dataset construction in more detail.

\subsection{Models}
\label{sec:experiments_models}

We use Qwen3 4B Thinking \citep{yang2025qwen3technicalreport} as a FOL parser in our main setup and Gemma 3 27B \citep{gemmateam2025gemma3technicalreport} as a translator. Since Qwen models occasionally return unparseable output, typically due to output token limits, we implement a retry mechanism with temperature sampling for these cases. Appendix~\ref{sec:appendix_models} provides more details on the models and the inference setup.

\section{Results}
\label{sec:results}

\subsection{Overall Scores}
\label{sec:main_results}

\begin{table}[t]
\centering
\small
\begin{tabular}{cccccc}
\toprule
Subtask & Acc.\hspace{1pt}$\uparrow$ & F1\hspace{1pt}$\uparrow$ & CE\hspace{1pt}$\downarrow$ & CS\hspace{1pt}$\uparrow$ & Rank \\
\midrule
1 & 95.29 & --    & 3.21 & 39.08 & 19th of 35 \\
2 & 97.37 & 96.84 & 3.30 & 39.49 & \pz 6th of 14 \\
3 & 93.75 & --    & 6.25 & 31.45 & \pz 8th of 13 \\
4 & 84.90 & 83.42 & 1.37 & 45.20 & \pz 4th of 15 \\
\bottomrule
\end{tabular}
\caption{Test set results from the official leaderboard. CS = combined score, Acc. = accuracy, F1 = premise F1, CE = total content effect, Rank = our system's rank in the official leaderboard.}
\label{tab:leaderboard}
\end{table}

The main leaderboard results are presented in Table~\ref{tab:leaderboard}. Despite the small parameter count of the models, our approach achieves competitive accuracy and a reasonably low content effect. The accuracy is preserved even in more complex subtasks, indicating that the approach is robust to irrelevant premises and multilingual variants. We present bootstrapped evaluation results with 95\% confidence intervals for both the validation and the test set in Table~\ref{tab:results_full} in the Appendix.

\subsection{Ablations}
\label{sec:ablations}

\begin{table}[t]
\centering
\small
\begin{tabular}{p{0.05cm}p{2.5cm}>{\centering\arraybackslash}p{0.65cm}>{\centering\arraybackslash}p{0.65cm}>{\centering\arraybackslash}p{0.65cm}>{\centering\arraybackslash}p{0.65cm}}
\toprule
\# & Setup & Acc.$\uparrow$ & F1\hspace{1pt}$\uparrow$ & CE\hspace{1pt}$\downarrow$ & CS\hspace{1pt}$\uparrow$ \\
\midrule
1 & Full                        & \textbf{95.83}   & -- & \pz 2.28   & 45.61 \\
  & - single-step & 95.81 & -- & \pz\textbf{2.14}   & \textbf{46.68} \\
  & - Prover9 format            & 85.97 & -- & 12.23 & 24.11 \\
  & - LLM prover         & 79.22 & -- & 20.21 & 19.58 \\
  & End-to-end                  & 84.17 & -- & 14.67 & 11.25 \\
\midrule
2 & Full                        & 91.17 & 87.77 & \pz 5.17 & \textbf{32.42} \\
  & - single-step & \textbf{91.76} & \textbf{89.68} & \pz 5.60 & 31.93 \\
  & - LLM retrieval       & 90.40 & 72.68 & \pz\textbf{4.76} & 30.43 \\
  & End-to-end                  & 83.40 & 64.11 & 14.92 & 19.65 \\
\midrule
3 & Full                 & \textbf{94.43} & -- & \pz\textbf{2.53} & \textbf{43.85} \\
  & - MT = Tiny Aya & 67.36 & -- & 23.62 & \pz 8.03 \\
  & - MT = Gemma3 4B & 87.18 & -- & \pz 5.14 & 15.96 \\
  & - MT = Qwen3 4B  & 79.75 & -- & \pz 7.09 & 13.41 \\
  & - MT = $\varnothing$  & 90.65 & -- & \pz 5.80 & 31.84 \\
  & End-to-end           & 83.01 & -- & 14.88 & 11.06 \\
\midrule
4 & Full                 & \textbf{90.67} & \textbf{85.10} & \pz 4.83 & 32.65 \\
  & - MT = $\varnothing$  & 85.85 & 78.14 & \pz \textbf{3.91} & \textbf{33.02} \\
  & End-to-end           & 81.39 & 60.72 & 17.83 & 18.11 \\
\bottomrule
\end{tabular}
\caption{Ablations of the pipeline modules (Qwen3 4B Thinking model). \# = subtask number. Metrics are explained in Table~\ref{tab:leaderboard}.}
\label{tab:pipeline_ablation}
\end{table}

Our ablations are primarily focused on the system components (Table~\ref{tab:pipeline_ablation}), but we also explore different model sizes and variants (Table~\ref{tab:model_ablation} in the Appendix). Table~\ref{tab:pipeline_ablation_full} in the Appendix includes 95\% confidence intervals obtained by bootstrap resampling. Prompt templates for all ablations are presented in Figures~\ref{fig:prompt_end_to_end}--\ref{fig:prompt_llm_prover} in the Appendix.

\paragraph{Zero-shot end-to-end classification} This ablation serves as a simple baseline, where the LLM is instructed to directly predict the validity of a natural language syllogism. For subtasks 2 and 4, we instruct the LLM to also provide indices of relevant premises. Table~\ref{tab:pipeline_ablation} shows a decrease of around 10 points in accuracy, more than 20 points lower F1, and a sharp increase in content effect across all subtasks.

\paragraph{Parsing directly to Prover9 syntax} To validate parsing through the intermediate LaTeX format, we design an ablation where the LLM is instructed to parse propositions directly to Prover9 syntax. As the results in Table \ref{tab:pipeline_ablation} show, the scores obtained with direct parsing are significantly worse, with a more than 10 point decrease in accuracy and a large increase in content effect. Our manual analysis reveals that 18\% formulas contain syntax errors, most of them caused by including invalid LaTeX commands in the Prover9 syntax, or adding extra parentheses.\footnote{In fact, syntax errors are even more prevalent in the raw outputs for this setup. However, we apply a regex-based cleanup (e.g. removing characters such as ``;'') for a fair comparison with our main setup.}

\paragraph{Single-step parsing} We compare the multi-step FOL formalization used in our main setup with a single-step formalization, where the entire syllogism is translated into FOL in a single inference call and generated in JSON format (see Figure~\ref{fig:prompt_single_step} in the Appendix). As the results show, both approaches lead to comparable performance.

\paragraph{LLM as a prover} This ablation tests the extent to which the LLM is able to replace the prover. Given the premises and the conclusion translated to FOL formulas, the LLM is instructed to decide if the conclusion is valid or not. The prompt template is shown in Figure~\ref{fig:prompt_llm_prover} in the Appendix. The results show a significant decrease of more than 15 points in accuracy, suggesting that LLMs of this size perform well in logical translation but not in reasoning, especially given the content effects.

\paragraph{LLM-based premise retrieval} The symbolic approach to relevant premise identification used in subtasks 2 and 4 (see Section \ref{sec:system_description}) is compared with a simple LLM-based approach, where a model is instructed to identify all relevant premises in formalized syllogisms. The results show a significantly worse F1 score compared to the symbolic approach.

\paragraph{Direct multilingual parsing} This ablation removes the pre-translation module and applies the FOL parser directly to the original untranslated syllogisms. The results show a substantial decrease in accuracy, which could be explained by the limited multilingual capabilities of models in this size range.

\paragraph{MT models} We compare different LLMs applied to the pre-translation step. In addition to Gemma3 27B used in our submission, we test three smaller models: Gemma3 4B, Qwen3 4B Instruct, and Tiny Aya Global \cite{salamanca_tiny_2026}, a multilingual LLM with 3.35B parameters. The results show that pre-translating with LLMs in this size range leads to even worse results than skipping the pre-translation step.

\paragraph{Model sizes and variants} We compare two model sizes of the Qwen3 family: 4B and 30B-A3B (MoE with 3B active parameters) as well as the Instruct and Thinking variants of both LLMs. Interestingly, our results in Table~\ref{tab:model_ablation} in the Appendix show that the 4B model achieves significantly higher accuracy than the larger variant, while the Instruct and Thinking variants of the 4B category achieve comparable accuracy.

\section{Error Analysis}
\label{sec:error_analysis}

We analyze the errors of our system on subtask 1, both in the validation set and in the test set released after the competition. Our system made 21 and 10 errors in the validation set and test set, respectively. However, we find that most of them are caused by label errors in the dataset or specifics of the Aristotelian logic \cite{sep-aristotle-logic}, while only two and four (validation + test) are clear errors. In the following, we describe the identified error categories in more detail.

\paragraph{Label errors}
Ten of the 21 validation errors were caused by logically incorrect ground truth labels.
For example, the following example was marked in the dataset as invalid, despite the conclusion following from the premises by simple transitivity:
\begin{samplebox}
P$_1$: Anything which is an animal is a flightless thing. \\
P$_2$: Every single bird is an animal. \\
C: There are no birds that are not flightless.
\end{samplebox}
The organizers seem to have revised these types of mistakes in the test set, as we did not observe this type of error there. 
However, the following inconsistencies remained.

\paragraph{Ambiguities between logic types}
We assume the existential promises of Aristotelian logic, as confirmed by the organizers of the shared task.
However, aside from the existential import (see Appendix~\ref{sec:aristotelian_logic}), we do not implement any specifics of Aristotelian logic to keep the system more general. We explored Aristotle's procedure for assessing validity with figures and moods using LLMs \cite[Chapter 5.4]{sep-aristotle-logic}, but it did not meaningfully surpass our main FOL system.

Three samples were not valid syllogisms, as they did not follow the exact structure (Appendix \ref{sec:aristotelian_logic}), although they were logically valid in FOL. 
For example, ``Birds'' are mentioned in all premises and in the conclusion, making it an invalid syllogism:
\begin{samplebox}
P$_1$: There are no birds that can be called fish.\\
P$_2$: It is also true that every bird is a type of animal. \\
C: This has led to the conclusion that a portion of animals are birds.
\end{samplebox}
Another one of the samples was invalid as a syllogism, because it used ``part of'' as a predicate:
\begin{samplebox}
P$_1$: Every single ocean is a body of water. \\
P$_2$: Any area that is a sea is part of an ocean. \\
C: This means that all seas are bodies of water.
\end{samplebox}

The remaining 6 of the 9 ambiguity-based errors involved more complex formulations which were parsed directly into FOL instead of being classified into one of four the categories of Aristotelian moods (Appendix \ref{sec:aristotelian_logic}). Five of these errors were made in syllogisms where 
a premise starts with ``It is not true that any:''
\begin{samplebox}
P$_1$: It is not true that any flower is a plant. \\
P$_2$: All things that are roses are flowers. \\
C: Thus, some roses are not plants.
\end{samplebox}
We understand it as ``Not for all,'' or ``There are some X for which it does not hold that...''.
However, in Aristotelian syllogisms this would be interpreted as ``For all X it does not hold that...'' (universal negative).
Similarly, if "something" is used, it should be interpreted as particular (``Some living organisms are mammals'') and not universal (``every'') as in FOL:
\begin{samplebox}
P$_1$: Something that is a living organism is a mammal. \\
P$_2$: The entire set of people is composed of living organisms. \\
C: It is the case that all people are mammals.
\end{samplebox}

In the test set, 6 of the 10 errors were based on logic type ambiguities.

\paragraph{Clear errors} Overall, we consider only 2 of the 21 errors on the validation set and 4 of the 10 errors on the test set to be clear errors caused by our method.
These errors were caused by Qwen 3 producing an incorrect FOL parse. For example, in the following sample, the conclusion was parsed into
$\exists x (\text{electronicdevice}(x) \land \text{noncomputer}(x))$, when the last predicate should have been $\lnot\text{computer}(x)$ to be compatible with the previous premises:
\begin{samplebox}
P$_1$: A computer is never a tablet. \\
P$_2$: All of the tablets are, without exception, electronic devices. \\
C: A subset of electronic devices is not composed of computers.
\end{samplebox}

\section{Analysis of the Combined Score}

We observe a large CS drop in competition submissions with near-perfect accuracy and a large CS variance, which prompted us to analyze the reliability of the CS metric.
Our analysis below shows that the CS metric over-amplifies random artifacts stemming from isolated errors.

\paragraph{Empirical analysis}
To assess the statistical significance of the content effects measured in the competition, we compare the competition results to those which would result from an unbiased model: one which classifies each of the four groups with the same accuracy $a \in [0, 1]$.
While such a model should not exhibit any content effect, as the difference of different group accuracies should be zero, we will show that some will be observed nonetheless.
To that end, we simulate predictions for the subtask 1 test set by choosing a model accuracy $a \sim U(0.5, 1)$ and assigning each sample a correct label with probability $a$.
We plot the measured scores of these simulations in Figure \ref{fig:empirical_analysis}. 
We see that the content effect of the simulated unbiased model not only shows a high variance, but is generally non-zero, with the exception of a perfect model.
Since the measured accuracies for each sample group are scaled binomial variables, the high variance in the measured content effects can be simply explained by the small size of the test set (191 samples, roughly 48 per group).
We can also see that most subtask 1 submissions match the simulated unbiased results -- only four of the submissions show a content effect significantly different from an unbiased model.

\begin{figure}[t]
    \centering
    \includegraphics[width=\linewidth]{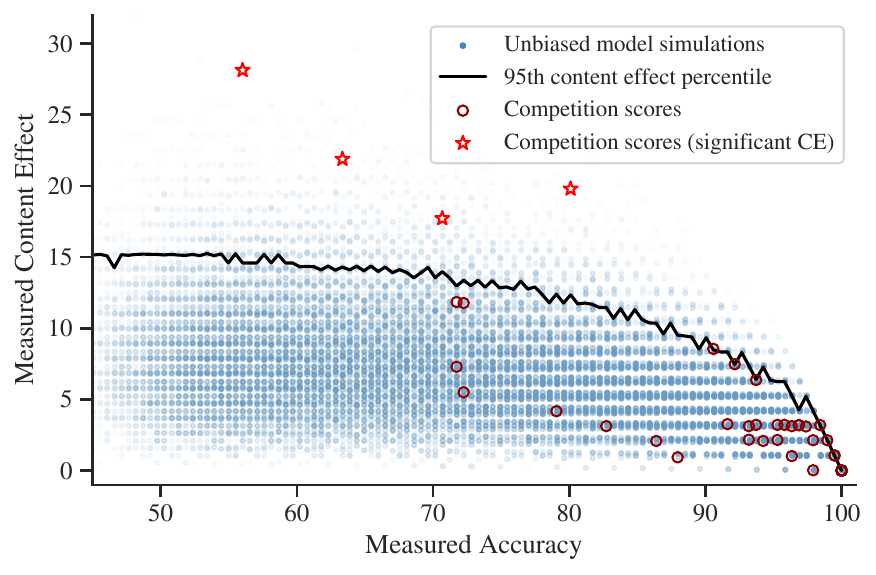}
    \caption{We simulate subtask 1 results of an unbiased model without content effect (\textcolor{plotblue}{blue}).
    We compare these runs against the competition submissions on the same test set (\textcolor{plotred}{red}, \textcolor{plotdarkred}{dark red}).
    By plotting the empirical 95th percentile of the unbiased model's measured content effect (\textbf{black}), we show that only four submissions exhibit a significant content effect (\textcolor{plotred}{red star}).
    }
    \label{fig:empirical_analysis}
\end{figure}

\begin{figure}[t]
    \centering
    \includegraphics[width=\linewidth]{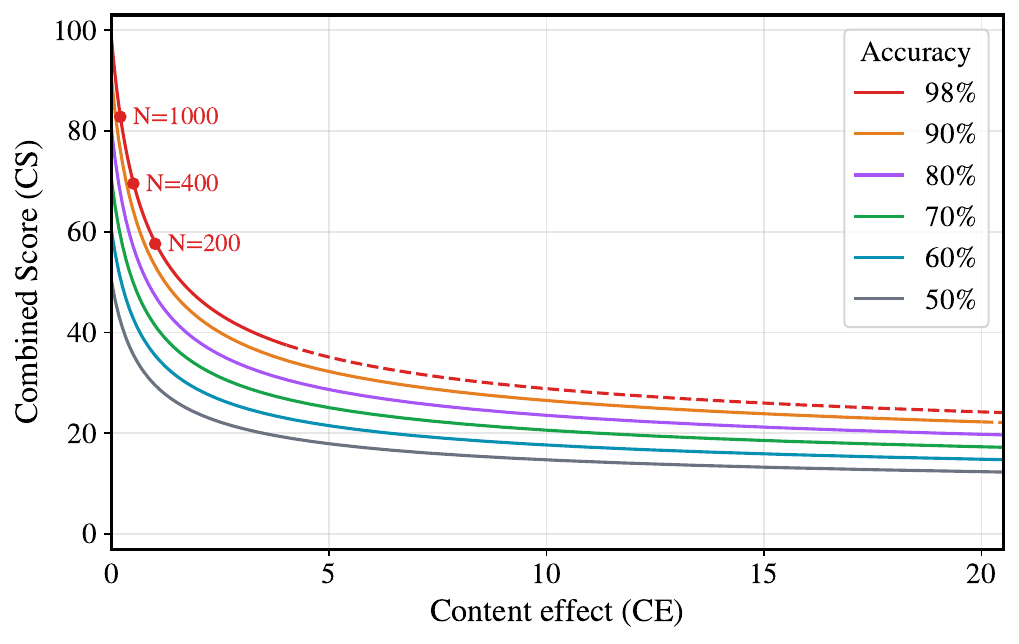}
    \caption{Sensitivity of the combined score (CS) metric with respect to changes in content effect (CE), showing the following properties: (1) CS is most sensitive when CE values are small; (2) systems with higher accuracy are penalized more; (3) flipping a single correct prediction can have a dramatic effect on CS.}
    \label{fig:score_vs_content_effect}
\end{figure}

\paragraph{Theoretical analysis}

We derive the expected content effect for an unbiased model, which is non-zero even when the group accuracies are equal:
\begin{proposition}\label{prop:avg_ce}
    For an unbiased model with accuracy $a \in [0,1]$, the expected measured content effect on a subtask 1 dataset with $N$ samples in each group is 
    $\mathbb{E}[\text{CE} | a] \approx 200 \sqrt{\frac{a(1-a)}{\pi N}}$.
\end{proposition}

See Appendix~\ref{app:score} for a proof of Proposition \ref{prop:avg_ce}.
From our experiments, this estimate is already tight for the size of the subtask 1 test set.
We should expect a content effect of $\approx 2.3$ for a $ 98\%$-accurate unbiased model, which already lowers its CS to less than half that of a perfect model.
Furthermore, the derivative of the $\mathbb{E}[\text{CE} | a]$ approximation approaches $-\infty$ as accuracy approaches $100\%$. 
While this is not accurate for the precise $\mathbb{E}[\text{CE} | a]$ as the Gaussian approximation used in Appendix~\ref{app:score} is not appropriate for values of $a$ close to~1, it nonetheless provides a reason for the sharp CS drop in task submissions with slightly lower than perfect accuracy.

Figure~\ref{fig:score_vs_content_effect} illustrates the sensitivity of the combined score with respect to the content effect for simulated systems with different accuracies. To demonstrate the concrete impact of this sensitivity, we simulate a high-accuracy system with no content effect and flip a single correct prediction. Even with a sample size $N=1000$ -- substantially larger than the competition's test set -- the combined score drops by more than 15 points.

\section{Conclusion}
\label{sec:conclusion}

This paper presents our system submitted to SemEval-2026 Task 11: Disentangling Content and Formal Reasoning in Large Language Models. We find that while small LLMs (Qwen 3 4B) demonstrate limited capabilities in syllogistic reasoning and high content effect, these models can be useful when integrated with symbolic tools, such as a first-order logic prover. Our experiments show the advantage of parsing syllogisms to LaTeX formulas, a format well-represented in LLMs' training data, rather than using arbitrary parser syntax. We also find that LLMs of this size are not sufficient for multilingual reasoning. Finally, we conduct an analysis of the main ranking metric and show that its sensitivity, combined with a small sample size, can overly penalize systems with high accuracy and low content effect.

\section*{Limitations}
\label{sec:limitations}

In our experiments, we did not evaluate the impact of fine-tuning LLMs. 
As our approach is a pipeline, any errors that occur in earlier stages will propagate through to later stages.
The Aristotelian syllogism has a specific structure which may complicate the formal mapping to FOL logic. 
It is unclear whether the parsing of natural language into FOL creates genuine reasoning with limited expressive power of syllogisms.
The generality of the presented approach and its evaluation may be limited: the output of small models may be sensitive to prompt variations, and the data and metrics may not reliably distinguish genuine content effect from statistical noise.

\section*{Acknowledgments}

This work was was funded by the European Union (ERC, NG-NLG, 101039303). It was additionally supported by the National Recovery Plan funded project MPO 60273/24/21300/21000 CEDMO 2.0 NPO, Charles University Research Centre program
No. 24/SSH/009, the project Human-centred AI for a Sustainable and Adaptive Society (reg. no.: CZ.02.01.01/00/23\_025/0008691), co-funded by the European Union, and Charles University SVV project number 260 821. It used resources of the LINDAT/CLARIAH-CZ Research Infrastructure (Czech Ministry of Education, Youth, and Sports project No. LM2018101).

%\bibliography{anthology,custom}
% Custom bibliography entries only
\bibliography{custom}

\appendix

\section{Expected Content Effect Analysis}
\label{app:score}

In this section, we analyze the expected value of the measured content effect for an unbiased model.
Let $A_g \in [0, 100]$ denote the classification accuracy for each group $g$.
The content effect score ${\text{CE}}$ is defined as:
\begin{align}
    C_\text{intra} &= \frac{|A_{\text{v},\text{p}} - A_{\text{v},\lnot\text{p}}| + |A_{\lnot\text{v},\text{p}} - A_{\lnot\text{v},\lnot\text{p}}|}{2}\\
    C_\text{inter} &= \frac{|A_{\text{v},\text{p}} - A_{\lnot\text{v},\text{p}}| + |A_{\text{v},\lnot\text{p}} - A_{\lnot\text{v},\lnot\text{p}}|}{2}\\
    {\text{CE}} &= \frac{1}{2}(C_\text{intra} + C_\text{inter}) \label{eq:content_effect}
\end{align}

Now, we focus on Proposition \ref{prop:avg_ce}, which states that
for an unbiased model with accuracy $a \in [0,1]$, the expected measured content effect on a sub-task 1 dataset with $N$ samples in each group is:
\begin{equation}
    \mathbb{E}[\text{CE} | a] \approx 200 \sqrt{\frac{a(1-a)}{\pi N}}
\end{equation}

\begin{proof}
    The sub-task 1 datasets consist of 4 groups, on which the content effect is measured as shown in Equation \ref{eq:content_effect}.
    For each group $g$, a correct classification can be modeled as a Bernoulli random variable, so the count of correct classifications is distributed binomially: $C_g \sim B(N, a)$.
    With enough samples, we can model the measured accuracy $A_g$ for that group as normally distributed:
    \begin{equation}
        A_g = \frac{100C_g}{N} \overset{\mathcal{L}}{\longrightarrow}\mathcal{N}\left(100a, \frac{100^2a(1-a)}{N}\right)
    \end{equation}
    Since the difference of the accuracies of the two groups would be roughly normally distributed as well, now with zero mean and doubled variance, the absolute value of that difference is distributed by the folded normal distribution.
    The mean of a folded normal distribution $\mathcal{F}$ is known: 
    \begin{equation}
        \mu_{\mathcal{F}} = \sigma\sqrt{\frac{2}{\pi}}e^{-\frac{\mu}{2\sigma^2}} + \mu\left(1 - 2\Phi\left(-\frac{\mu}{\sigma}\right)\right)
    \end{equation}
    In our case $\mu=0$, and $\sigma^2=2\frac{100^2a(1-a)}{N}$,
    so the expected value of the absolute difference of accuracies in different groups amounts to:
    \begin{equation}
        \mathbb{E}[|A_{g_1} - A_{g_2}|] \approx 200 \sqrt{\frac{a(1-a)}{\pi N}}
    \end{equation}
    Averaging the four absolute differences to obtain $\text{CE}$ as in Equation \ref{eq:content_effect} leaves the expected value unchanged.
\end{proof}

\section{Assumptions of Aristotelian logic}
\label{sec:aristotelian_logic}

As the shared task is based on Aristotelian logic, it assumes existential import: a universal proposition implies the corresponding particular proposition \cite[Chapter 5.2]{sep-aristotle-logic}. For example, ``all A are B'' implies ``some A is B''.

However, existential import is not an assumption of formal logic provers like Prover9, which are based on Boolean logic. We address this by synthesizing the required particular propositions through a regex-based extraction of the necessary predicates from the existing propositions.

\paragraph{Valid sylogisms}
For a syllogism to be in valid form, it must include exactly two premises and a conclusion. Each premise contains a middle term that is not in conclusion. The other term in a premise is a predicate of a conclusion (major premise) or a subject of a conclusion (minor premise) \cite[Chapter 5]{sep-aristotle-logic}. 
Subjects can be individual (e.g. Socrates) or universal (e.g. human) and predicates only universal \cite[Chapter 4]{sep-aristotle-logic}. 

Sentences can take one of four forms for universal subjects: universal affirmative (``Every S is P''), universal negative (``No S is P''), particular affirmative (``Some S is P'') or particular negative (``Some S is not P''), and two additional forms for individual subjects: ``S is P'' and ``S is not P'' \cite[Chapter 4.3]{sep-aristotle-logic}.

\begin{table*}[t]
\centering
\small
\begin{tabular}{ll}
\toprule
Name & Ollama tag \\
\midrule
Qwen3 4B thinking & \texttt{qwen3:4b-thinking-2507-fp16} \\
Qwen3 4B instruct & \texttt{qwen3:4b-instruct-2507-fp16} \\
Qwen3 30B-A3B thinking & \texttt{qwen3:30b-a3b-thinking-2507-q8\_0} \\
Qwen3 30B-A3B instruct & \texttt{qwen3:30b-a3b-instruct-2507-q8\_0} \\
Tiny Aya Global 3.35B & \texttt{hf.co/CohereLabs/tiny-aya-global-GGUF:BF16} \\
Gemma 3 27B & \texttt{gemma3:27b-it-fp16} \\
Gemma 3 4B & \texttt{gemma3:4b-it-fp16} \\
\bottomrule
\end{tabular}
\caption{List of all LLMs used in our experiments.}
\label{tab:models}
\end{table*}

\section{Validation Data}
\label{sec:appendix_data}

In this section, we provide details on the construction of our validation data for all four subtasks.

\paragraph{Subtask 1} We use the first 500 training set examples as our validation data for this subtask. This subset also serves as a basis for constructing the validation sets for the other three subtasks.

\paragraph{Subtask 2} We draw irrelevant premises from a pool of the remaining 460 training examples. For each validation example, we select 3--5 irrelevant premises from topically related pool examples: these are identified by shared nouns that appear in at least two propositions of a syllogism, extracted through POS tagging. This selection respects plausibility labels, drawing irrelevant premises only from plausible pool examples for plausible validation examples, while premises for implausible examples may come from any pool example. The selected irrelevant premises are combined with the original premises, de-duplicated, and shuffled.

\paragraph{Subtask 3} The multilingual data for subtask 3 are synthesized by translating the validation set to the languages included in the shared task: English (en), German (de), Spanish (es), French (fr), Italian (it), Dutch (nl), Portuguese (pt), Russian (ru), Chinese (zh), Swahili (sw), Bengali (bn), Telugu (te). For each validation example, we randomly sample a target language and instruct Claude Opus 4.6 to translate the syllogism to the language.

\paragraph{Subtask 4} Since this subtasks includes both multilingual inputs and irrelevant premises, we combine the approaches for subtasks 2 and 3 to synthesize the corresponding validation set. Specifically, we first sample irrelevant premises and then translate this augmented dataset to the target languages.

\section{Models and Parameters}
\label{sec:appendix_models}

We run all our experiments with greedy decoding (the temperature parameter is set to 0). The only exception are retries for the FOL parser, which are triggered by unparseable outputs. In these cases, we repeat the response generation with a temperature of 0.6. The context size is set to 16384 tokens.
We use the Ollama\footnote{\url{https://ollama.com/}} platform for inference. Table~\ref{tab:models} presents all models used in our experiments and the corresponding Ollama tags.

\section{Prompt Templates}
\label{sec:appendix}

\subsection{Submitted System}

Prompt templates for the FOL parser (see Section~\ref{sec:system_description}) used in our main setup are presented in Figures~\ref{fig:prompt_default} and \ref{fig:prompt_initial}. Figures~\ref{fig:prompt_translation}--\ref{fig:prompt_translation_feedback} show the prompt templates for the translation, self-evaluation, and correction steps of the Translator component.

\subsection{Ablations}

Figures~\ref{fig:prompt_end_to_end}--\ref{fig:prompt_llm_retrieval} show prompt templates for all ablation experiments described in Section~\ref{sec:ablations}.

\section{Detailed Results}
\label{sec:appendix_results}

\subsection{Main Results}
\label{sec:detailed_results_main}

The results on both the validation and the test set with 95\% confidence intervals are presented in Table~\ref{tab:results_full}.

\begin{table*}[ht]
\centering
\small
\begin{tabular}{llcccc}
\toprule
Split & Task & Accuracy $\uparrow$ & F1 Premises $\uparrow$ & Content Effect $\downarrow$ & Combined Score $\uparrow$ \\
\midrule
Validation & 1 & 95.83 [94.20, 97.40]   & -- & 2.28 [0.70, 4.38]   & 45.61 [35.35, 63.04] \\
           & 2 & 91.17 [88.60, 93.60] & 87.77 [83.97, 91.33] & 5.17 [2.34, 7.98]   & 32.42 [27.76, 40.52] \\
           & 3 & 90.65 [88.20, 93.21] & -- & 5.80 [2.43, 9.56]   & 31.84 [26.52, 40.85] \\
           & 4 & 85.85 [82.80, 88.80] & 78.14 [73.25, 82.74] & 3.91 [1.21, 7.34]   & 33.02 [26.09, 45.46] \\
\midrule
Test & 1 & 94.77 [91.62, 97.38] & -- & 5.81 [2.50, 10.87] & 33.45 [26.91, 42.99] \\
     & 2 & 97.35 [94.74, 99.47] & 96.91 [93.18, 100.00] & 3.93 [1.04, 8.17] & 39.98 [29.33, 57.84] \\
     & 3 & 93.71 [90.10, 96.88] & -- & 7.17 [3.33, 11.76] & 30.84 [25.44, 39.28] \\
     & 4 & 84.75 [79.67, 90.10] & 83.36 [75.98, 90.46] & 6.03 [1.75, 12.02] & 29.85 [22.98, 42.09] \\
\bottomrule
\end{tabular}
\caption{Detailed results on validation and test sets with 95\% confidence intervals obtained with bootstrap resampling.}
\label{tab:results_full}
\end{table*}

\subsection{Ablations}
\label{sec:detailed_results_ablations}

\begin{table*}[ht]
\centering
\small
\begin{tabular}{llcccc}
\toprule
Task\hspace{-2mm} & Setup & Accuracy $\uparrow$ & F1 Premises $\uparrow$ & Content Effect $\downarrow$ & Combined Score $\uparrow$ \\
\midrule
1 & Full                        & \textbf{95.83} [94.20, 97.40]   & -- & 2.28 [0.70, 4.38]   & 45.61 [35.35, 63.04] \\
  & - single-step & 95.81 [94.00, 97.40] & -- & \textbf{2.14} [0.57, 4.24]   & \textbf{46.68} [35.84, 66.39] \\
  & - Prover9 format            & 85.97 [82.86, 88.72] & -- & 12.23 [8.99, 16.08] & 24.11 [21.77, 26.78] \\
  & - LLM prover         & 79.22 [75.40, 82.80] & -- & 20.21 [16.49, 23.87] & 19.58 [18.04, 21.27] \\
  & End-to-end                  & 84.17 [80.80, 87.40] & -- & 14.67 [11.31, 18.31] & 11.25 [10.31, 12.35] \\
\midrule
2 & Full                        & 91.17 [88.60, 93.60] & 87.77 [83.97, 91.33] & 5.17 [2.34, 7.98] & \textbf{32.42} [27.76, 40.52] \\
  & - single-step & \textbf{91.76} [89.40, 94.20] & \textbf{89.68} [86.19, 92.82] & 5.60 [2.84, 8.48] & 31.93 [27.75, 39.12] \\
  & - LLM retrieval       & 90.40 [87.80, 93.00] & 72.68 [67.01, 77.89] & \textbf{4.76} [1.90, 7.89] & 30.43 [25.31, 39.65] \\
  & End-to-end                  & 83.40 [80.00, 86.60] & 64.11 [58.40, 69.83] & 14.92 [11.45, 18.80] & 19.65 [17.54, 22.02] \\
\midrule
3 & Full                 & \textbf{94.43} [92.40, 96.20] & -- & \textbf{2.53} [0.67, 5.26] & \textbf{43.85} [32.99, 63.61] \\
  & - MT = Tiny Aya Global \hspace{-3mm} & 67.36 [63.40, 71.21] & -- & 23.62 [18.90, 28.27] & 8.03 [7.35, 8.75] \\
  & - MT = Gemma3 4B & 87.18 [84.00, 90.01] & -- & 5.14 [1.88, 8.72] & 15.96 [13.17, 21.34] \\
  & - MT = Qwen3 4B  & 79.75 [76.35, 83.17] & -- & 7.09 [2.31, 12.93] & 13.41 [10.84, 17.97] \\
  & - MT = $\varnothing$  & 90.65 [88.20, 93.21] & -- & 5.80 [2.43, 9.56] & 31.84 [26.52, 40.85] \\
  & End-to-end           & 83.01 [79.60, 86.20] & -- & 14.88 [11.24, 18.81] & 11.06 [10.08, 12.19] \\
\midrule
4 & Full                 & \textbf{90.67} [88.20, 93.00] & \textbf{85.10} [80.88, 89.01] & 4.83 [2.02, 8.46] & 32.65 [27.07, 41.93] \\
  & - MT = $\varnothing$  & 85.85 [82.80, 88.80] & 78.14 [73.25, 82.74] & \textbf{3.91} [1.21, 7.34] & \textbf{33.02} [26.09, 45.46] \\
  & End-to-end           & 81.39 [78.00, 84.80] & 60.72 [54.92, 66.33] & 17.83 [13.75, 22.04] & 18.11 [16.25, 20.24] \\
\bottomrule
\end{tabular}
\caption{Detailed results of the component ablations with including 95\% confidence intervals obtained with bootstrap resampling.}
\label{tab:pipeline_ablation_full}
\end{table*}

Table~\ref{tab:model_ablation} shows the comparison of the instruct and thinking variants, as well as different model sizes.

\begin{table*}[ht]
\centering
\small
\begin{tabular}{llcccc}
\toprule
Task & Model & Accuracy $\uparrow$ & F1 Premises $\uparrow$ & Content Effect $\downarrow$ & Combined Score $\uparrow$ \\
\midrule
1 & Qwen3 4B Thinking  & \textbf{95.83} [94.2, 97.4]   & -- & 2.28 [0.70, 4.38]   & \textbf{45.61} [35.35, 63.04] \\
  & Qwen3 4B Instruct  & 94.83 [92.6, 96.8]   & -- & \textbf{2.27} [0.62, 4.80]   & 45.60 [34.16, 64.35] \\
  & Qwen3 30B Thinking & 91.37 [88.8, 93.8]   & -- & 4.25 [1.56, 7.20]   & 35.51 [28.92, 46.74] \\
  & Qwen3 30B Instruct & 77.79 [74.0, 81.4]   & -- & 21.16 [17.40, 25.41] & 19.02 [17.55, 20.58] \\
\midrule
2 & Qwen3 4B Thinking  & \textbf{91.17} [88.60, 93.60] & \textbf{87.77} [83.97, 91.33] & 5.17 [2.34, 7.98]   & 32.42 [27.76, 40.52] \\
  & Qwen3 4B Instruct  & 90.65 [88.00, 93.20] & 86.37 [82.59, 90.23] & \textbf{4.79} [2.09, 8.19]   & \textbf{32.92} [27.45, 41.89] \\
\midrule
3 & Qwen3 4B Thinking  & \textbf{90.65} [88.20, 93.21] & -- & \textbf{5.80} [2.43, 9.56]   & \textbf{31.84} [26.52, 40.85] \\
  & Qwen3 4B Instruct  & 89.02 [86.40, 91.60] & -- & 6.39 [2.75, 10.65]  & 30.42 [25.37, 38.38] \\
\midrule
4 & Qwen3 4B Thinking  & \textbf{85.85} [82.80, 88.80] & \textbf{78.14} [73.25, 82.74] & \textbf{3.91} [1.21, 7.34]   & \textbf{33.02} [26.09, 45.46] \\
  & Qwen3 4B Instruct  & 84.13 [80.80, 87.20] & 72.92 [68.10, 77.92] & 4.45 [1.34, 8.24]   & 30.28 [24.10, 41.77] \\
\bottomrule
\end{tabular}
\caption{Comparison of model size and variants with 95\% confidence intervals obtained with bootstrap resampling.}
\label{tab:model_ablation}
\end{table*}

\begin{figure*}[h]
\small
\begin{tcolorbox}[
    colback=white,
    colframe=bordercolor,
    arc=3mm,
    boxrule=0.75pt,
    width=0.95\textwidth,
    left=5pt,
    right=5pt,
    top=5pt,
    bottom=5pt,
]
\ttfamily
\color{textgray}
\setlength{\parindent}{0pt}
\setlength{\parfillskip}{0pt plus 1fil}
\emergencystretch=3em
\sloppy
\raggedright

You are given a proposition in natural language. Your task is to convert the proposition to first-order logic. For logical operators, use latex symbols: \textbackslash forall, \textbackslash exists, \textbackslash land, \textbackslash lor, \textbackslash neg, \textbackslash rightarrow. If possible, parse the proposition to a formula consisting of two atomic formulas, each of them a unary predicate. Each unique predicate should be represented by a lowercase (or camel-case) word. For a context, you are also given previous propositions already translated to first-order logic. Make sure you use predicate mapping from previous responses. Generate the final formula in the boxed format.

\vspace{5pt}

Previous propositions:

\vspace{2pt}

\{previous\_propositions\}

\vspace{5pt}

Proposition: \{proposition\}

\end{tcolorbox}
\caption{Prompt template for the FOL parser. For each natural language proposition, the LLM is instructed to translate it to FOL representation, and is given the mapping from natural language to FOL representation for all previously translated propositions in the syllogism.}
\label{fig:prompt_default}
\end{figure*}

\begin{figure*}[h]
\small
\begin{tcolorbox}[
    colback=white,
    colframe=bordercolor,
    arc=3mm,
    boxrule=0.75pt,
    width=0.95\textwidth,
    left=5pt,
    right=5pt,
    top=5pt,
    bottom=5pt,
]
\ttfamily
\color{textgray}
\setlength{\parindent}{0pt}
\setlength{\parfillskip}{0pt plus 1fil}
\emergencystretch=3em
\sloppy
\raggedright

You are given a proposition in natural language. Your task is to convert the proposition to first-order logic. For logical operators, use latex symbols: \textbackslash forall, \textbackslash exists, \textbackslash land, \textbackslash lor, \textbackslash neg, \textbackslash rightarrow. If possible, parse the proposition to a formula consisting of two atomic formulas, each of them a unary predicate. Each unique predicate should be represented by a lowercase (or camel-case) word. Generate the final formula in the boxed format.

\vspace{5pt}

Proposition: \{proposition\}

\end{tcolorbox}
\caption{Prompt template for the FOL parser used to parse the initial proposition of a syllogism.}
\label{fig:prompt_initial}
\end{figure*}

\begin{figure*}[h]
\small
\begin{tcolorbox}[
    colback=white,
    colframe=bordercolor,
    arc=3mm,
    boxrule=0.75pt,
    width=0.95\textwidth,
    left=5pt,
    right=5pt,
    top=5pt,
    bottom=5pt,
]
\ttfamily
\color{textgray}
\setlength{\parindent}{0pt}
\setlength{\parfillskip}{0pt plus 1fil}
\emergencystretch=3em
\sloppy
\raggedright

Translate the following syllogism to English. Translated syllogism will be used as an input to a FOL parser. Make sure that your translation is unambiguous and easy to parse and understand. Some propositions may be nonsensical, but you should still preserve their meaning. Output ONLY the translation, nothing else.

\vspace{5pt}

Text to translate:

\vspace{2pt}

\{syllogism\}

\vspace{5pt}

English translation:

\end{tcolorbox}
\caption{Prompt template for translation used in subtasks 3 and 4.}
\label{fig:prompt_translation}
\end{figure*}

\begin{figure*}[h]
\small
\begin{tcolorbox}[
    colback=white,
    colframe=bordercolor,
    arc=3mm,
    boxrule=0.75pt,
    width=0.95\textwidth,
    left=5pt,
    right=5pt,
    top=5pt,
    bottom=5pt,
]
\ttfamily
\color{textgray}
\setlength{\parindent}{0pt}
\setlength{\parfillskip}{0pt plus 1fil}
\emergencystretch=3em
\sloppy
\raggedright

You are a translation quality evaluator. Your task is to verify if the following translation of a syllogism is correct.

\vspace{5pt}

Original text:

\vspace{2pt}

\{formatted\_original\}

\vspace{5pt}

Translation:

\vspace{2pt}

\{translation\}

\vspace{5pt}

Determine whether the translation preserves the meaning of each proposition, and whether there any mistranslations or omissions.

\vspace{5pt}

Provide your verdict as a JSON object with exactly these fields:

- ``feedback'': explanation of errors if incorrect, or confirmation that translation is correct

- ``correct'': true if the translation is acceptable, false otherwise

\vspace{5pt}

Your response MUST end with the JSON object on its own line, formatted as:

\{``feedback'': ``<your feedback>'', ``correct'': <true or false>\}

\end{tcolorbox}
\caption{Prompt template for translation self-evaluation.}
\label{fig:prompt_translation_evaluation}
\end{figure*}

\begin{figure*}[h]
\small
\begin{tcolorbox}[
    colback=white,
    colframe=bordercolor,
    arc=3mm,
    boxrule=0.75pt,
    width=0.95\textwidth,
    left=5pt,
    right=5pt,
    top=5pt,
    bottom=5pt,
]
\ttfamily
\color{textgray}
\setlength{\parindent}{0pt}
\setlength{\parfillskip}{0pt plus 1fil}
\emergencystretch=3em
\sloppy
\raggedright

Translate the following syllogism to English. Translated syllogism will be used as an input to a FOL parser. Make sure that your translation is unambiguous and easy to parse and understand. Some propositions may be nonsensical, but you should still preserve their meaning. Output ONLY the translation, nothing else.

\vspace{5pt}

Text to translate:

\vspace{2pt}

\{syllogism\}

\vspace{5pt}

A previous translation attempt was incorrect. Here is the feedback:

\vspace{2pt}

\{feedback\}

\vspace{5pt}

Please provide a corrected translation.

\vspace{5pt}

English translation:

\end{tcolorbox}
\caption{Prompt template for translation with feedback provided by the self-evaluation step.}
\label{fig:prompt_translation_feedback}
\end{figure*}

\begin{figure*}[h]
\small
\begin{tcolorbox}[
    colback=white,
    colframe=bordercolor,
    arc=3mm,
    boxrule=0.75pt,
    width=0.95\textwidth,
    left=5pt,
    right=5pt,
    top=5pt,
    bottom=5pt,
]
\ttfamily
\color{textgray}
\setlength{\parindent}{0pt}
\setlength{\parfillskip}{0pt plus 1fil}
\emergencystretch=3em
\sloppy
\raggedright

You are a logic expert specializing in formal reasoning and categorical syllogisms. Assess the logical validity of a syllogism regardless of its real-world plausibility. Carefully examine the premises and the conclusion to determine whether the conclusion necessarily follows from the premises using formal logical structure. Identify any fallacies, such as undistributed middle, invalid contraposition, or missing conclusions. If the conclusion is logically entailed by the premises, output 'true'; otherwise, output 'false'.

\vspace{5pt}

Generate your output as a JSON object with the following fields.

\vspace{2pt}

\{ \\
  \hspace{6pt}"reasoning": "...", \\
  \hspace{6pt}"valid": "..."        \# note: the value you produce must be 'true' or 'false'" \\
\}

\vspace{5pt}

[[ \#\# syllogism \#\# ]]

\vspace{2pt}

\{syllogism\}

\vspace{5pt}

Respond with a JSON object in the following order of fields: 'reasoning', then 'valid' (must be formatted as a valid JSON bool).

\end{tcolorbox}
\caption{Prompt template for the zero-shot end-to-end baseline.}
\label{fig:prompt_end_to_end}
\end{figure*}

\begin{figure*}[h]
\small
\begin{tcolorbox}[
    colback=white,
    colframe=bordercolor,
    arc=3mm,
    boxrule=0.75pt,
    width=0.95\textwidth,
    left=5pt,
    right=5pt,
    top=5pt,
    bottom=5pt,
]
\ttfamily
\color{textgray}
\setlength{\parindent}{0pt}
\setlength{\parfillskip}{0pt plus 1fil}
\emergencystretch=3em
\sloppy
\raggedright

You are a logic expert specializing in formal reasoning and categorical syllogisms. Assess the logical validity of a syllogism regardless of its real-world plausibility. Carefully examine the premises and the conclusion to determine whether the conclusion necessarily follows from the premises using formal logical structure. Identify any fallacies, such as undistributed middle, invalid contraposition, or missing conclusions. If the conclusion is logically entailed by the premises, output 'true'; otherwise, output 'false'.

\vspace{5pt}

If the syllogism is valid, also identify which premises are relevant (i.e., necessary) for the conclusion to follow. List them as a 0-indexed array of premise indices.

\vspace{5pt}

Generate your output as a JSON object with the following fields.

\vspace{2pt}

\{ \\
  \hspace{6pt}"reasoning": "...", \\
  \hspace{6pt}"valid": "..."        \# note: the value you produce must be 'true' or 'false'" \\
  \hspace{6pt}"relevant\_premises": [0, 1, ...]  \# only if valid is 'true'; 0-indexed premise indices \\
\}

\vspace{5pt}

[[ \#\# syllogism \#\# ]]

\vspace{2pt}

\{syllogism\}

\vspace{5pt}

Respond with a JSON object in the following order of fields: 'reasoning', then 'valid' (must be formatted as a valid JSON bool), then 'relevant\_premises' (only if valid is 'true').

\end{tcolorbox}
\caption{Prompt template for the zero-shot end-to-end baseline with relevant premise identification.}
\label{fig:prompt_end_to_end_retrieval}
\end{figure*}

\begin{figure*}[h]
\small
\begin{tcolorbox}[
    colback=white,
    colframe=bordercolor,
    arc=3mm,
    boxrule=0.75pt,
    width=0.95\textwidth,
    left=5pt,
    right=5pt,
    top=5pt,
    bottom=5pt,
]
\ttfamily
\color{textgray}
\setlength{\parindent}{0pt}
\setlength{\parfillskip}{0pt plus 1fil}
\emergencystretch=3em
\sloppy
\raggedright

You are given a proposition in natural language. Your task is to convert the proposition to first-order logic. Express the statement in Prover9 syntax (an automated theorem prover library). Below you have a list of the most common logical operations, their symbol in Prover9 and an example of a formula.

\vspace{5pt}

\begin{tabular}{l l l}
Operation & Prover9 & Example \\
negation & - & (-p) \\
disjunction & | & (p | q | r) \\
conjunction & \& & (p \& q \& r) \\
implication & -> & (p -> q) \\
backward implication & <- & (p <- q) \\
equivalence & <-> & (p <-> q) \\
universal quantification & all & (all x all y P(x,y)) \\
existential quantification & exists & (exists x exists y P(x,y)) \\
combinations of the above & ... & (exists x (P(x)) \& exists x (P(x) \& Q(x))) \\
\end{tabular}

\vspace{5pt}

If possible, parse the proposition to a formula consisting of two atomic formulas, each of them a unary predicate. Each predicate should be represented by a single uppercase alphabetic symbol. Note that single-letter lowercase names: x y z u v w p q r must be individual variables in Prover9. So, ``p(x)'' is incorrect because ``p'' is an illegal predicate name, but ``P(x)'' is correct. For a context, you are also given previous propositions already translated to first-order logic. Make sure you use predicate mapping from previous responses. Generate the final formula in the boxed format.

\vspace{5pt}

Previous propositions:

\vspace{2pt}

\{previous\_propositions\}

\vspace{5pt}

Proposition: \{proposition\}

\end{tcolorbox}
\caption{Prompt template for the ablation where the FOL parser parses propositions directly to the Prover9 syntax.}
\label{fig:prompt_direct_parsing_default}
\end{figure*}

\begin{figure*}[h]
\small
\begin{tcolorbox}[
    colback=white,
    colframe=bordercolor,
    arc=3mm,
    boxrule=0.75pt,
    width=0.95\textwidth,
    left=5pt,
    right=5pt,
    top=5pt,
    bottom=5pt,
]
\ttfamily
\color{textgray}
\setlength{\parindent}{0pt}
\setlength{\parfillskip}{0pt plus 1fil}
\emergencystretch=3em
\sloppy
\raggedright

You are given a proposition in natural language. Your task is to convert the proposition to first-order logic. Express the statement in Prover9 syntax (an automated theorem prover library). Below you have a list of the most common logical operations, their symbol in Prover9 and an example of a formula.

\vspace{5pt}

\begin{tabular}{l l l}
Operation & Prover9 & Example \\
negation & - & (-p) \\
disjunction & | & (p | q | r) \\
conjunction & \& & (p \& q \& r) \\
implication & -> & (p -> q) \\
backward implication & <- & (p <- q) \\
equivalence & <-> & (p <-> q) \\
universal quantification & all & (all x all y P(x,y)) \\
existential quantification & exists & (exists x exists y P(x,y)) \\
combinations of the above & ... & (exists x (P(x)) \& exists x (P(x) \& Q(x))) \\
\end{tabular}

\vspace{5pt}

If possible, parse the proposition to a formula consisting of two atomic formulas, each of them a unary predicate. Each predicate should be represented by a single uppercase alphabetic symbol. Note that single-letter lowercase names: x y z u v w p q r must be individual variables in Prover9. So, ``p(x)'' is incorrect because ``p'' is an illegal predicate name, but ``P(x)'' is correct. Generate the final formula in the boxed format.

\vspace{5pt}

Proposition: \{proposition\}

\end{tcolorbox}
\caption{Prompt template for the initial step in the ablation where the FOL parser parses propositions directly to the Prover9 syntax.}
\label{fig:prompt_direct_parsing_initial}
\end{figure*}

\begin{figure*}[h]
\small
\begin{tcolorbox}[
    colback=white,
    colframe=bordercolor,
    arc=3mm,
    boxrule=0.75pt,
    width=0.95\textwidth,
    left=5pt,
    right=5pt,
    top=5pt,
    bottom=5pt,
]
\ttfamily
\color{textgray}
\setlength{\parindent}{0pt}
\setlength{\parfillskip}{0pt plus 1fil}
\emergencystretch=3em
\sloppy
\raggedright

You are given a syllogism in natural language, consisting of \{num\_premises \} premises and a conclusion. You task is to convert premises and conclusion of the syllogism to first-order logic. For logical operators, use LaTeX symbols: \textbackslash forall, \textbackslash exists, \textbackslash land, \textbackslash lor, \textbackslash neg, \textbackslash rightarrow. Make sure to use consistent mapping from natural language to FOL predicates. If possible, parse each proposition to a formula consisting of two atomic formulas, each of them a unary predicate. Each unique predicate should be represented by a lowercase (or camel-case) word. Do not focus on validity of the syllogism, only parse it to FOL. Generate the output with natural language propositions and first-order logic formulas in the JSON format: [\{``proposition'': ``...'', ``fol\_formula'': ``...''\}, ...]

\vspace{5pt}

Syllogism:

\vspace{2pt}

\{syllogism\}

\end{tcolorbox}
\caption{Prompt template for the FOL parser used for single-step parsing. Unlike the multi-step setup, the LLM is instructed to translate the entire syllogism to a FOL representation and generate the output in JSON format.}
\label{fig:prompt_single_step}
\end{figure*}

\begin{figure*}[h]
\small
\begin{tcolorbox}[
    colback=white,
    colframe=bordercolor,
    arc=3mm,
    boxrule=0.75pt,
    width=0.95\textwidth,
    left=5pt,
    right=5pt,
    top=5pt,
    bottom=5pt,
]
\ttfamily
\color{textgray}
\setlength{\parindent}{0pt}
\setlength{\parfillskip}{0pt plus 1fil}
\emergencystretch=3em
\sloppy
\raggedright

You are given an argument consisting of two premises and a conclusion, expressed in first-order logic. Your task is to analyze the formulas and decide whether the conclusion logically follows from the premises.

\vspace{5pt}

Do not focus on notational details of the first-order logic representation or on the specific names of predicates. Focus only on the logical structure and validity of the argument.

\vspace{5pt}

Premises:

\vspace{5pt}

\{premises\}

\vspace{5pt}

Conclusion:

\vspace{5pt}

\{conclusion\}

\vspace{5pt}

Perform analysis and give the final answer in boxed format:

- \textbackslash boxed\{true\} if the conclusion is logically valid

- \textbackslash boxed\{false\} if it is not

\end{tcolorbox}
\caption{Prompt template for the LLM prover.}
\label{fig:prompt_llm_prover}
\end{figure*}

\begin{figure*}[h]
\small
\begin{tcolorbox}[
    colback=white,
    colframe=bordercolor,
    arc=3mm,
    boxrule=0.75pt,
    width=0.95\textwidth,
    left=5pt,
    right=5pt,
    top=5pt,
    bottom=5pt,
]
\ttfamily
\color{textgray}
\setlength{\parindent}{0pt}
\setlength{\parfillskip}{0pt plus 1fil}
\emergencystretch=3em
\sloppy
\raggedright

You are given a syllogism that has been determined to be valid. Your task is to identify which premises are actually relevant (necessary) for the conclusion to follow.

\vspace{5pt}

Premises:

\vspace{2pt}

\{premises\}

\vspace{5pt}

Conclusion: \{conclusion\}

\vspace{2pt}

Return ONLY a JSON array of 0-based indices of the relevant premises. For example: [0, 2]

\end{tcolorbox}
\caption{Prompt template for the LLM-based relevant premise retrieval.}
\label{fig:prompt_llm_retrieval}
\end{figure*}

\end{document}